\documentclass{article}
\usepackage[utf8]{inputenc}

\usepackage{diagbox}
\usepackage{makecell}
\usepackage{booktabs}
\usepackage{amsmath}
\usepackage{amsfonts}
\usepackage{amssymb}
\usepackage[numbers,sort]{natbib}
\usepackage{amsthm}
\usepackage{url,ifthen}
\usepackage{enumitem}
\usepackage{srcltx}
\usepackage{multirow}
\usepackage{boxedminipage}
\usepackage[margin=1.1in]{geometry}
\usepackage{nicefrac}
\usepackage{xspace}
\usepackage{graphicx}
\usepackage[dvipsnames,table,xcdraw]{xcolor}
\usepackage{subfigure}
\usepackage{setspace}

\newtheorem{theorem}{Theorem}[section]

\newtheorem{lemma}[theorem]{Lemma}
\newtheorem{corollary}[theorem]{Corollary}

\theoremstyle{definition}

\DeclareMathOperator*{\argmin}{arg\,min}

\newcommand{\R}{\mathbb{R}}
\newcommand{\N}{\mathbb{N}}
\newcommand{\E}{\mathbb{E}}

\newcommand{\KL}{\mathrm{KL}}

\title{A Note on Zeroth-Order Optimization on the Simplex}
\author{Tijana Zrnic$^*$ ~~~ Eric Mazumdar$^\dagger$\\ \\
$^*$University of California, Berkeley\\
$^\dagger$California Institute of Technology}

\begin{document}

\maketitle

\begin{abstract}
We construct a zeroth-order gradient estimator for a smooth function defined on the probability simplex. The proposed estimator queries the simplex only. We prove that projected gradient descent and the exponential weights algorithm, when run with this estimator instead of exact gradients, converge at a $\mathcal O(T^{-1/4})$ rate. 
\end{abstract}

\section{Introduction}
Resource allocation, mechanism design, load balancing, strategic classification, and many other problems with economic incentives require optimizing an objective function on the simplex. The simplex constraint can describe a fixed load to be distributed~\cite{NoRegretRouting}, a constraint on the budget to be allocated~\cite{GameSimplex}, or a distribution over actions individuals can take \cite{zrnic2021leads}; see \citet{bomze2002regularity} and \citet{de2008complexity} for a survey on simplex optimization with applications. Moreover, these optimization problems often have to be solved using only \emph{zeroth-order} feedback, i.e., function evaluations at different points on the simplex, as gradient feedback is not easily obtainable.

There are a number of methods for zeroth-order optimization (e.g.~\cite{flaxman2004online, shamir2013complexity, hazan2014bandit, bubeck2017kernel}), typically relying on approximating function gradients using only function evaluations. However, these methods usually require querying the function at a noisy point in a ball around the current iterate. Applying such a strategy on the simplex would violate the simplex constraint, due to the low dimensionality of the simplex. Indeed, existing algorithms (see, e.g., \cite{flaxman2004online} or \cite{shamir2013complexity}) often assume that the constraint set in $\R^d$ contains a ball in $\R^d$ with small enough radius, a condition not satisfied by the simplex. This obstacle motivates designing and studying new algorithms for zeroth-order optimization on the simplex.

In this note we construct a one-point zeroth-order gradient estimator that only queries the simplex and prove convergence of standard optimization methods when run with this estimator. In Section \ref{sec:estimator} we present the estimator and prove a bound on its expected error for smooth objectives. Then, in Section \ref{sec:algos} we show that projected gradient descent and the exponential weights algorithm converge at a $\mathcal O(T^{-1/4})$ rate when run with this estimator.

To set the problem up more formally, suppose that an objective function $f$ can only be queried on the $d$-dimensional probability simplex, denoted $\Delta_d$. We show that querying $f$ at a convex combination of a point $x\in\Delta_d$ and a Dirichlet noise vector---a combination guaranteed to lie on the simplex---approximately recovers the gradient of $f$ at $x$.

In particular, we prove that
\begin{equation}
\label{eq:estimator_informal}
C \E[f((1-\delta)x + \delta u)P_d u] \approx P_d \nabla f(x),
\end{equation}
where $P_d = I_d - \frac{1}{d}\mathbf{1}_d\mathbf{1}_d^\top$, $u$ is sampled from a Dirichlet distribution, $\delta\in(0,1)$ is a tuning parameter, and $C$ is an explicit normalizing factor. Here, $I_d$ denotes the $d$-dimensional identity matrix and $\mathbf{1}_d$ denotes the $d$-dimensional all-ones vector. Moreover, we show that having access to $P_d \nabla f(x)$ preserves all first-order information about $f$, which will allow us to approximate running standard gradient-based algorithms on $f$ using only zeroth-order information.

\section{Zeroth-order gradient estimator}
\label{sec:estimator}

In what follows we formalize statement \eqref{eq:estimator_informal} in Theorem \ref{thm:error_bound}; then, in Corollary \ref{corollary:1st_order} we show that our zeroth-order gradient estimator provides a good first-order approximation.

The quality of our estimator relies on $f$ being $\beta$-smooth. Formally, we assume
\begin{equation}
\tag{A1}
\label{eq:smoothness}
\|\nabla f(x) - \nabla f(y)\|\leq \beta \|x-y\|,
\end{equation}
for all $x,y\in\Delta_d$. We use $\|\cdot\|$ to denote the $\ell_2$-norm throughout.

We let $\mathrm{Dir}(\alpha \mathbf{1}_d)$ denote the $d$-dimensional Dirichlet distribution with parameter $(\alpha,\dots,\alpha)$, for $\alpha>0$.

\begin{theorem}[Error of gradient estimator]
\label{thm:error_bound}
Suppose that $f:\Delta_d\rightarrow\R$ is $\beta$-smooth \eqref{eq:smoothness}. Let $u\sim \mathrm{Dir}(\alpha\mathbf{1}_d)$. Then, for all $x\in\Delta_d$ and $\delta\in(0,1)$, it holds that
$$ \left\|\E \left[ \frac{1}{\delta} f((1-\delta)x + \delta u) P_d u\right] - \frac{1}{d(\alpha d + 1)} P_d \nabla f(x) \right\| \leq 2\beta \delta,$$
where $P_d = I_d - \frac{1}{d}\mathbf{1}_d\mathbf{1}_d^\top$.
\end{theorem}

\begin{proof}
By Lemma \ref{lemma:Pd_properties} (c), we know $\E P_d u = 0$; hence
\begin{align*}
    &\E \left[ \frac{1}{\delta} f((1-\delta)x + \delta u)P_d u - \frac{1}{d(\alpha d + 1)} P_d \nabla f(x) \right]\\
    &\quad = \E \left[ \frac{1}{\delta} f((1-\delta)x + \delta u)P_d u -  \frac{1}{d(\alpha d + 1)} P_d \nabla f(x) - \frac{1}{\delta} f(x) P_d u +   P_d u \nabla f(x)^\top x  \right].
\end{align*}
Now we use the fact that the covariance matrix of the Dirichlet distribution $\mathrm{Dir}(\alpha \mathbf{1}_d)$ is given by
$$\E u u^\top =  \frac{1}{d(\alpha d + 1)} P_d,$$
and so $P_d = d(\alpha d + 1)\E uu^\top$. Moreover, by Lemma \ref{lemma:Pd_properties} (a), we have $P_d = d(\alpha d + 1)P_d \E uu^\top$. This further gives
\begin{align*}
    &\E \left[ \frac{1}{\delta} f((1-\delta)x + \delta u)P_d u - \frac{1}{d(\alpha d + 1)} P_d \nabla f(x) - \frac{1}{\delta} f(x) P_d u +  P_d u \nabla f(x)^\top x  \right]\\
    &\quad = \E \left[ \frac{1}{\delta} f((1-\delta)x + \delta u)P_d u -  P_d uu^\top \nabla f(x) - \frac{1}{\delta} f(x) P_d u +   P_d u \nabla f(x)^\top x  \right]\\
    &\quad =  \E \left[P_d u \left( \frac{1}{\delta} f((1-\delta)x + \delta u) - \frac{1}{\delta} f(x) -   \nabla f(x)^\top (u-x)   \right)  \right].
\end{align*}
Now we apply Jensen's inequality to get
\begin{align*}
&\left\|\E \left[ \frac{1}{\delta} f((1-\delta)x + \delta u)P_d u - \frac{1}{d(\alpha d + 1)} P_d \nabla f(x) \right]\right\|\\
&\quad \leq \E \left[\left\|P_d u\right\| \left| \frac{1}{\delta} f((1-\delta)x + \delta u) - \frac{1}{\delta} f(x) -   \nabla f(x)^\top (u-x)   \right| \right].
\end{align*}
By Lemma \ref{lemma:Pd_properties} (d), we know $\|P_d u\| \leq 1$ almost surely. Moreover, by Lemma \ref{lemma:smoothness}, we have
$$\left| \frac{1}{\delta} f((1-\delta)x + \delta u) - \frac{1}{\delta} f(x) -   \nabla f(x)^\top (u-x)   \right| \leq \frac{\beta\delta}{2}\|u-x\|^2.$$
Therefore,
$$\left\|\E \left[ \frac{1}{\delta} f((1-\delta)x + \delta u)P_d u - \frac{1}{d(\alpha d + 1)} P_d \nabla f(x) \right]\right\| \leq \frac{\beta\delta}{2}\E \|u - x\|^2 \leq 2\beta\delta ,$$
which follows since $\sup_{x,x'\in\Delta_d} \|x - x'\| \leq 2$.
\end{proof} 

Now we show that the gradient estimator from Theorem \ref{thm:error_bound} preserves all first-order information about $f$, which will allow approximating first-order optimization methods using only zeroth-order feedback.

\begin{corollary}[First-order approximation]
\label{corollary:1st_order}
Suppose that $f:\Delta_d\rightarrow\R$ is $\beta$-smooth \eqref{eq:smoothness}. Consider the gradient estimator $\tilde g = \frac{d(\alpha d +1)}{\delta} f((1-\delta)x + \delta u)P_d u$, where $u\sim \mathrm{Dir}(\alpha \mathbf{1}_d)$. Then,
$$\left| \E \tilde g^\top (x'-x) - \nabla f(x)^\top (x'-x)\right| \leq 4\beta\delta d(\alpha d + 1).$$
\end{corollary}

\begin{proof}
By Lemma \ref{lemma:Pd_properties} (b) we have
\begin{align*}
    \nabla f(x)^\top (x' - x) = \nabla f(x)^\top P_d (x' - x).
\end{align*}
Using this, together with Theorem \ref{thm:error_bound} and the Cauchy-Schwarz inequality, we have
\begin{align}
\label{eq:grad}
    |\E \tilde g^\top (x' - x) - \nabla f(x)^\top (x' - x)| \leq 2\beta\delta d(\alpha d + 1)  \|x-x'\| \leq 4\beta\delta d(\alpha d + 1).
\end{align}
which follows since $\sup_{x,x'\in\Delta_d} \|x-x'\|\leq 2$.
\end{proof}

\section{Zeroth-order algorithms}
\label{sec:algos}

We prove convergence of standard first-order optimization methods, when gradient information is replaced with approximate gradient information obtained via the zeroth-order estimator from Theorem~\ref{thm:error_bound}. In particular, only assuming smoothness, in Lemma \ref{lemma:convergence} and Lemma \ref{lemma:convergence_ew} we show convergence of zeroth-order versions of projected gradient descent and the exponential weights algorithm in terms of the first-order suboptimality gap, $\max_{x^*\in\Delta_d} \nabla f(x_t)^\top (x_t - x^*)$. When $f$ is additionally convex, this in turn implies that the average iterate converges to the optimum, which we formalize in Theorem \ref{thm:avg_iterate}.

\subsection{Projected ``gradient'' descent}

Consider the update:
\begin{equation}
\label{eq:PGD}
\tag{PGD}
x_{t+1} = \Pi_{\Delta_d}(y_{t+1}):= \Pi_{\Delta_d}(x_t - \eta \tilde g_t), \text{ where } \tilde g_t = \frac{d(\alpha d +1)}{\delta} f((1-\delta)x_t + \delta u_t) P_d u_t.
\end{equation}
Here, $u_1,u_2,\dots \stackrel{\mathrm{i.i.d.}}{\sim}\mathrm{Dir}(\alpha \mathbf{1}_d)$, $\eta>0$ is a step size, and $\Pi_{\Delta_d}$ denotes a projection onto the simplex.

\begin{lemma}[Projected gradient descent]
\label{lemma:convergence}
Suppose that $f:\Delta_d\rightarrow\R$ is $\beta$-smooth \eqref{eq:smoothness}. Then, zeroth-order projected gradient descent \eqref{eq:PGD} with $\eta \propto \frac{T^{-3/4}}{d(\alpha d +1)}$ and $\delta \propto T^{-1/4}$ satisfies
$$\max_{x^*\in\Delta_d} \frac{1}{T} \sum_{t=1}^T \E \nabla f(x_t)^\top (x_t - x^*) = \mathcal{O}\left(T^{-1/4}\right).$$
\end{lemma}

\begin{proof}
Let $x^*$ be an arbitrary point on the simplex $\Delta_d$.
Then, by Corollary \ref{corollary:1st_order},
\begin{align}
\label{eq:grad_bound}
    \E \nabla f(x_t)^\top (x_t - x^*) &\leq  \E \tilde g_t^\top (x_t - x^*) +  4\beta\delta d(\alpha d + 1).
\end{align}
Note that we can write
$$\tilde g_t = \frac{1}{\eta} (x_t - y_{t+1}).$$
Plugging this into Eq.~\eqref{eq:grad_bound}, we have
\begin{align*}
    \E \nabla f(x_t)^\top (x_t - x^*) &\leq  \frac{1}{\eta} \E (x_t - y_{t+1})^\top (x_t - x^*) + 4\beta\delta d(\alpha d + 1)\\
    &=  \frac{1}{2\eta} \left(\E \|x_t - y_{t+1}\|^2 + \E\|x_t - x^*\|^2 - \E \|y_{t+1} - x^*\|^2\right) + 4\beta\delta d(\alpha d + 1)\\
    &\leq  \frac{1}{2\eta} \left(\E \|x_t - y_{t+1}\|^2 + \E \|x_t - x^*\|^2 - \E \|x_{t+1} - x^*\|^2\right) + 4\beta\delta d(\alpha d + 1),
\end{align*}
where we use the fact that the projected iterate must be closer to $x^*$ than the unprojected one, $\|x_{t+1} - x^*\| \leq \|y_{t+1} - x^*\|$, which follows by convexity of the simplex.

Now we sum up both sides over $t\in\{1,\dots,T\}$ and divide by $T$:
\begin{align*}
\frac{1}{T}\sum_{t=1}^T \E \nabla f(x_t)^\top (x_t - x^*) &\leq  \frac{1}{2\eta T} \left( \sum_{t=1}^T\E \|x_t - y_{t+1}\|^2 + \|x_1 - x^*\|^2 - \E \|x_{T+1} - x^*\|^2\right) + 4\beta\delta d(\alpha d + 1),
\end{align*}
which follows by a telescoping-sum argument.

Since $f$ is smooth and defined on the probability simplex, there must exist a value $B>0$ such that $\sup_{x\in\Delta_d} |f(x)|\leq B$. This in turn implies $\|x_t - y_{t+1}\| \leq \frac{\eta d(\alpha d + 1)}{\delta} B$ for all $t$. Therefore,
\begin{align*}
\frac{1}{T}\sum_{t=1}^T \E \nabla f(x_t)^\top (x_t - x^*) \leq \frac{B^2 \eta d^2(\alpha d + 1)^2}{2\delta^2} +  \frac{2}{\eta T} + 4\beta\delta d(\alpha d +1),
\end{align*}
where we use $\|x_1 - x^*\|\leq 2$ to bound the second term.

Finally, setting $\eta \propto \frac{T^{-3/4}}{d(\alpha d +1)}$ and $\delta \propto T^{-1/4}$ yields a rate of
$$\max_{x^*\in\Delta_d} \frac{1}{T}\sum_{t=1}^T  \E \nabla f(x_t)^\top (x_t - x^*) = \mathcal O\left(T^{-1/4}\right).$$
\end{proof}

\subsection{Exponential weights}

We also consider a zeroth-order version of the exponential weights algorithm in order to remove the need for the potentially costly projection onto the simplex. In particular, each entry $i\in \{1,\dots, d\}$ is updated as:
\begin{equation}
\label{eq:exp_weights}
\tag{EW}
 x_{t+1,i}=\frac{x_{t,i}\exp(-\eta \tilde g_{t, i})}{Z_t}, \text{ where } \tilde g_{t}= \frac{d(\alpha d + 1)}{\delta} f((1-\delta)x_t + \delta u_t)P_d u_t.
 \end{equation}
Here, $u_1,u_2,\dots\stackrel{\mathrm{i.i.d.}}{\sim} \text{Dir}(\alpha \mathbf{1}_d)$, $Z_t=\sum_{i=1}^d x_{t,i}\exp(-\eta \tilde g_{t, i})$, and $\eta>0$ is a step size.

\begin{lemma}[Exponential weights]
\label{lemma:convergence_ew}
Suppose that $f:\Delta_d\rightarrow\R$ is $\beta$-smooth \eqref{eq:smoothness}. Then, the exponential weights update \eqref{eq:exp_weights} with $\eta \propto \frac{T^{-3/4}}{d(\alpha d +1)}$ and $\delta \propto T^{-1/4}$ satisfies
$$\max_{x^*\in\Delta_d} \frac{1}{T} \sum_{t=1}^T \E \nabla f(x_t)^\top (x_t - x^*) = \mathcal{O}\left(T^{-1/4}\right).$$
\end{lemma}

\begin{proof}
Fix any $x^* \in \Delta_d$; then we have
\begin{align}
    \KL(x^*||x_t)-\KL(x^*||x_{t+1})&=\sum_{i=1}^d x_i^*\log{\frac{x_{t+1,i}}{x_{t,i}}}\nonumber\\
    &=\sum_{i=1}^d x_i^*\left(-\eta \tilde g_{t,i}-\log{Z_t} \right)\nonumber\\
    &=-\eta \tilde g_t^\top x^* -\log{Z_t}. \label{eq:ineq1}
\end{align}
Focusing on the second term, we find that
\begin{align}
    \log{Z_t}&=\log{\sum_{i=1}^d x_{t,i}\exp(-\eta \tilde g_{t, i})}\nonumber\\
    &\le \log{\sum_{i=1}^d x_{t,i}\left(1-\eta \tilde g_{t,i}+\eta^2 \tilde g_{t,i}^2\right)}\nonumber\\
    &\le \log{\left(1-\eta x_t^\top \tilde g_t +\eta^2\|\tilde g_t\|^2\right)}\nonumber\\
    &\le -\eta  x_t^\top \tilde g_t +\eta^2\|\tilde g_t\|^2 \label{eq:ineq2},
\end{align}
where the first inequality follows because $\exp(x)\le 1+x+x^2$ for $x\le 1$ and $|\eta \tilde g_{t,i}|$ is less than $1$ for large enough $T$; the last inequality follows since
$\log x\le x-1$. Putting Eq.~\eqref{eq:ineq1} and Eq.~\eqref{eq:ineq2} together, we find that
\begin{align*}
    \eta (x_t -x^*)^\top \tilde g_t &\le \KL(x^*||x_t)-\KL(x^*||x_{t+1})+ \eta^2\|\tilde g_t\|^2.
\end{align*}
This in turn implies
\begin{align*}
\eta ( x_t -x^*)^\top \nabla f(x_t) &\le \eta ( x_t -x^*)^\top (\nabla f(x_t)-\tilde g_t) +\KL(x^*||x_t)-\KL(x^*||x_{t+1})+ \eta^2\|\tilde g_t\|^2.
\end{align*}
Taking an expectation on both sides above, we find that
\begin{align*}
    \eta \E (x_t -x^*)^\top \nabla f(x_t) &\le \eta \E (x_t -x^*)^\top (\nabla f(x_t)-\tilde g_t) +\E \KL(x^*||x_t)-\E \KL(x^*||x_{t+1})+ \eta^2\E \|\tilde g_t\|^2\\
    &\le 4\eta \beta\delta d(\alpha d + 1) + \frac{\eta^2 B^2 d^2(\alpha d +1)^2}{\delta^2} + \E \KL(x^*||x_t)-\E\KL(x^*||x_{t+1}),
\end{align*}
where we use Corollary \ref{corollary:1st_order}, as well as the fact that there must exist a value $B>0$ such that $\sup_{x\in\Delta_d} |f(x)|\leq B$, given that $f$ is smooth and defined on the probability simplex.

Summing over $t\in\{1,\dots,T\}$ and dividing by $\eta T$ gives
\begin{align*}
\frac{1}{T}\sum_{t=1}^T  \E \nabla f(x_t)^\top (x_t - x^*) \le 4\beta\delta d(\alpha d+1) + \frac{\eta B^2 d^2(\alpha d+1)^2}{\delta^2} + \frac{\KL(x^*||x_1)}{\eta T}.
\end{align*}
Choosing $\eta \propto \frac{T^{-3/4}}{d(\alpha d + 1)}$ and $\delta \propto T^{-1/4}$ yields the final result.
\end{proof}

\subsection{Convergence to optima under convexity}

Lemma \ref{lemma:convergence} and Lemma \ref{lemma:convergence_ew} prove that the usual first-order suboptimality gap for convex problems decreases at a $\mathcal O(T^{-1/4})$ rate; moreover, these technical lemmas hold even if $f$ is not convex. When $f$ is convex, however, this directly implies that the average iterate converges to the optimum.

\begin{theorem}
\label{thm:avg_iterate}
Suppose that $f:\Delta_d\rightarrow\R$ is $\beta$-smooth \eqref{eq:smoothness} and additionally convex. Let $x^* = \argmin_{x\in\Delta_d} f(x)$. Then, both zeroth-order projected gradient descent \eqref{eq:PGD} and the exponential weights update \eqref{eq:exp_weights} with $\eta\propto \frac{T^{-3/4}}{d(\alpha d + 1)}$ and $\delta\propto T^{-1/4}$ satisfy
$$\E f(\bar x_T) - f(x^*) = \mathcal O\left(T^{-1/4}\right),$$
where $\bar x_T = \frac{1}{T}\sum_{t=1}^T x_t$.
\end{theorem}

\begin{proof}
Since $f$ is convex, we know
\begin{align*}
    f(x_t) - f(x^*) \leq \nabla f(x_t)^\top (x_t - x^*).
\end{align*}
Therefore, we have the following bound:
$$\frac{1}{T} \sum_{t=1}^T \E f(x_t) - f(x^*) \leq \frac{1}{T} \sum_{t=1}^T \E \nabla f(x_t)^\top (x_t - x^*).$$
By Jensen's inequality we know $\frac{1}{T} \sum_{t=1}^T f(x_t)  \geq f\left(\frac{1}{T} \sum_{t=1}^T x_t\right)$, and so we have
$$ \E f\left(\frac{1}{T} \sum_{t=1}^T x_t\right) - f(x^*) \leq \frac{1}{T} \sum_{t=1}^T \E \nabla f(x_t)^\top (x_t - x^*).$$
Finally, invoking Lemma \ref{lemma:convergence} or Lemma \ref{lemma:convergence_ew} (depending on the algorithm), yields
$$\E f\left(\frac{1}{T} \sum_{t=1}^T x_t\right) - f(x^*) = \mathcal O\left(T^{-1/4}\right),$$
which completes the proof.
\end{proof}

\section{Auxiliary lemmas}

\begin{lemma}
\label{lemma:Pd_properties}
Let
$$P_d = I_d - \frac{1}{d} \mathbf{1}_d\mathbf{1}_d^\top.
$$
Then,
\begin{itemize}
\item[(a)] $P_d^k = P_d$, for all $k\in \N$;
\item[(b)] $P_d (x-x') = x-x'$ for all $x,x'\in\Delta_d$;
\item[(c)] $\E[P_d u] = 0$ for $u\sim\mathrm{Dir}(\alpha\mathbf{1}_d)$, for all $\alpha>0$;
\item[(d)] $\sup_{x\in\Delta_d} \|P_d x\|\leq 1$.
\end{itemize}
\end{lemma}

\begin{proof}
\textit{(a)} We show $P_d^2 = P_d$, which implies $P_d^k = P_d$ for all $k\in\N$. This follows since
$$P_d^2 = \left(I_d - \frac{1}{d} \mathbf{1}_d\mathbf{1}_d^\top\right) \left(I_d - \frac{1}{d} \mathbf{1}_d\mathbf{1}_d^\top\right) = I_d - \frac{2}{d} \mathbf{1}_d\mathbf{1}_d^\top + \frac{\mathbf{1}_d^\top\mathbf{1}_d}{d^2} \mathbf{1}_d\mathbf{1}_d^\top = I_d - \frac{1}{d} \mathbf{1}_d\mathbf{1}_d^\top = P_d.$$

\textit{(b)} For all $x,x'\in\Delta_d$, we have 
$$P_d (x-x') = x-x' - \frac{1}{d} \mathbf{1}_d\mathbf{1}_d^\top(x-x') = x-x',$$
since $\mathbf{1}_d^\top x = \mathbf{1}_d^\top x' = 1$.

\textit{(c)} For all $\alpha>0$, the mean of the Dirichlet distribution is $\E[u] = \frac{1}{d} \mathbf{1}_d$; this in turn gives $$\E[P_d u] = \left(I_d - \frac{1}{d} \mathbf{1}_d\mathbf{1}_d^\top\right) \frac{1}{d} \mathbf{1}_d = \frac{1}{d} \mathbf{1}_d - \frac{\mathbf{1}_d^\top \mathbf{1}_d}{d^2} \mathbf{1}_d = 0.$$

\textit{(d)} For any $x\in\Delta_d$, we have
\begin{align*}
    \|P_d x\|^2 &= \left\|\left(I_d - \frac{1}{d} \mathbf{1}_d\mathbf{1}_d^\top\right) x\right\|^2\\
    &= \left\|x - \frac{1}{d} \mathbf{1}_d\right\|^2\\
    &= \|x\|^2 - \frac{2}{d} x^\top \mathbf{1}_d + \frac{1}{d^2} \|\mathbf{1}_d\|^2\\
    &= \|x\|^2 - \frac{2}{d} + \frac{1}{d}\\
    &\leq 1 - \frac{1}{d}\\
    &= \frac{d-1}{d},
\end{align*}
where the inequality follows due to $\|x\|\leq \|x\|_1 = 1$. Therefore, $\|P_dx\| \leq \sqrt{\frac{d-1}{d}} \leq 1$.
\end{proof}

\begin{lemma}
\label{lemma:smoothness}
Let $f:\mathcal{D}\rightarrow \R$ be $\beta$-smooth, i.e.
$$\|\nabla f(x) - \nabla f(y)\|\leq \beta \|x-y\|,$$
for all $x,y\in\mathcal{D}\subseteq \R^d$. Then, for all $x,y\in\mathcal{D}$:
$$|f(y) - f(x) - \nabla f(x)^\top (y-x)| \leq \frac{\beta}{2}\|x-y\|^2.$$
\end{lemma}

\begin{proof}
By the Cauchy-Schwarz inequality, the smoothness condition implies 
$$- \beta \|x-y\|^2 \leq (\nabla f(x) - \nabla f(y))^\top (x - y) \leq \beta \|x-y\|^2.$$
Rearranging, this expression is equivalent to
$$(\nabla f(x) - \nabla f(y) + \beta(x-y))^\top (x-y) \geq 0 \text{ and } (\beta (x-y) - (\nabla f(x) - \nabla f(y)))^\top (x-y) \geq 0.$$
By the monotone-gradient characterization of convexity, this means that both $f(x) + \frac{\beta}{2}\|x\|^2$ and $\frac{\beta}{2}\|x\|^2 - f(x)$ are convex. The first-order characterization of convexity for these two functions is exactly
$$f(y) \geq f(x) + \nabla f(x)^\top (y-x) - \frac{\beta}{2} \|x-y\|^2 \text{ and } f(y) \leq f(x) + \nabla f(x)^\top (y-x) + \frac{\beta}{2} \|x-y\|^2,$$
which directly implies the desired claim.
\end{proof}

\bibliographystyle{plainnat}
\bibliography{refs}

\end{document}